\def\year{2020}\relax
\documentclass[letterpaper]{article} %DO NOT CHANGE THIS
\usepackage{aaai20}  % DO NOT CHANGE THIS
\usepackage{times}  % DO NOT CHANGE THIS
\usepackage{helvet} % DO NOT CHANGE THIS
\usepackage{courier}  % DO NOT CHANGE THIS
\usepackage[hyphens]{url}  % DO NOT CHANGE THIS
\usepackage{graphicx} % DO NOT CHANGE THIS
\usepackage{graphicx}  % DO NOT CHANGE THIS

\usepackage[utf8]{inputenc}

\usepackage[]{todonotes}

\usepackage{amssymb}
\usepackage{amsmath}
\usepackage{amsthm}

\usepackage{booktabs}

\newtheorem{theorem}{Theorem}

\theoremstyle{definition}
\newtheorem{definition}{Definition}

\usepackage[nolist,nohyperlinks]{acronym}
\usepackage[inline]{enumitem}

\usepackage{pgf}
\usepackage{tikz}
\usetikzlibrary{arrows,automata}
\usetikzlibrary{shapes,snakes}
\usetikzlibrary{intersections}

\usepackage{mathtools,algorithm}
\usepackage[noend]{algpseudocode}
\usepackage[capitalise]{cleveref}

\theoremstyle{proposition}
\newtheorem{proposition}{Proposition}

\theoremstyle{corollary}
\newtheorem{corollary}{Corollary}

\crefname{algorithm}{Alg.}{Algs.}
\Crefname{algorithm}{Algorithm}{Algorithms}

\tikzstyle{line} = [draw, -latex']
\tikzstyle{hidden} = [ellipse, draw, text centered, inner sep=1pt]
\tikzstyle{obs} = [ellipse, draw, fill=gray!60, text centered, inner sep=1pt]
\tikzstyle{rv} = [draw, ellipse, inner sep=2pt]
\tikzstyle{pf} = [draw, rectangle, fill=gray]
\tikzstyle{pc} = [draw, rounded corners=15pt, align=center, minimum width=18mm, font=\normalsize, inner sep=3pt]
\tikzstyle{pc2} = [draw, rounded corners=8pt,align=center,minimum height=6mm,font=\normalsize,inner sep=2pt]

\tikzstyle{nhidden} = [draw=none,fill=none,ellipse, text centered, inner sep=1pt]
\tikzstyle{nobs} = [draw=none,fill=none,ellipse, fill=gray!60, text centered, inner sep=1pt]
\tikzstyle{nrv} = [draw=none,fill=none,ellipse, inner sep=2pt]
\tikzstyle{npf} = [draw=none,fill=none,rectangle]

\tikzstyle{ID} = [draw, circle, font=\normalsize]
\tikzstyle{INN} = [draw, circle, inner sep=1pt, fill=black]

\newcommand\pfs[8]{
  \node[pf, #1 of=#2, node distance=#3, xshift=-1mm, yshift=1mm](#6) {};
  \node[pf, #1 of=#2, node distance=#3, label=#4:{#5}](#7) {};
  \node[pf, #1 of=#2, node distance=#3, xshift=1mm, yshift=-1mm](#8) {};
}

\title{Taming Reasoning in Temporal Probabilistic Relational Models}

\author{Marcel Gehrke, Ralf Möller, and Tanya Braun\\
Institute of Information Systems, University of Lübeck, Lübeck  \\
\{gehrke, braun, moeller\}@ifis.uni-luebeck.de}

\begin{document}

\maketitle
% no keywords
\acrodef{SHR}[SHR]{Standard Health Record}
\acrodef{ehr}[EHR]{electronic health record}
\acrodef{FHIM}[FHIM]{Federal Health Information Model}
\acrodef{OHDSI}[OHDSI]{OMOP Common Data Model, the Observational Health Data Sciences and Informatics}
\acrodef{FHIR}[FHIR]{HL7’s FAST Healthcare Interoperability Resources}

\acrodef{jtree}[jtree]{junction tree}
%\acrodef{Plms}[PLMs]{Probabilistic logical models}
\acrodef{plms}[PLMs]{probabilistic logical models}

\acrodef{pdb}[PDB]{probabilistic database}

\acrodef{pf}[parfactor]{parametric factor}
\acrodef{lv}[logvar]{logical variable}
\acrodef{rv}[randvar]{random variable}
\acrodef{crv}[CRV]{counting randvar}
\acrodef{prv}[PRV]{parameterised randvar}

\acrodef{fodt}[FO dtree]{first-order decomposition tree}

\acrodef{fojt}[FO jtree]{first-order junction tree}
\acrodef{ljt}[LJT]{lifted junction tree algorithm}
\acrodef{ldjt}[LDJT]{lifted dynamic junction tree algorithm}
\acrodef{lve}[LVE]{lifted variable elimination}
\acrodef{ve}[VE]{variable elimination}

\acrodef{tam}[TAMe]{temporal approximate merging}

\acrodef{2tpm}[2TPM]{two-slice temporal parameterised model}
\acrodef{2tbn}[2TBN]{two-slice temporal bayesian network}

\acrodef{pm}[PM]{parameterised probabilistic model}
\acrodef{pdecm}[PDecM]{parameterised probabilistic decision model}

\acrodef{pdm}[PDM]{parameterised probabilistic dynamic model}
\acrodef{pddecm}[PDDecM]{parameterised probabilistic dynamic decision model}

\acrodef{dbn}[DBN]{dynamic Bayesian network}
\acrodef{bn}[BN]{Bayesian network}

\acrodef{dfg}[DFG]{dynamic factor graph}

\acrodef{dmln}[DMLN]{dynamic Markov logic network}
\acrodef{mln}[MLN]{Markov logic network}

\acrodef{rdbn}[RDBN]{relational dynamic Bayesian network}

\acrodef{meu}[MEU]{maximum expected utility}
\acrodef{mldn}[MLDN]{Markov logic decision network}

\acrodef{cep}[CEP]{complex event processing}

\acrodef{krp}[KRP]{keeping reasoning polynomial}

\begin{abstract}
	Evidence often grounds temporal probabilistic relational models over time, which makes reasoning infeasible.
	To counteract groundings over time and to keep reasoning polynomial by restoring a lifted representation, we present \ac{tam}, which incorporates 
	\begin{enumerate*}
		\item clustering for grouping submodels as well as
		\item statistical significance checks to test the fitness of the clustering outcome.
	\end{enumerate*} 
	In exchange for faster runtimes, \ac{tam} introduces a bounded error that becomes negligible over time.
	Empirical results show that \ac{tam} significantly improves the runtime performance of inference, while keeping errors small.
\end{abstract}
\acresetall	
% no keywords

\section{Introduction}\label{sec:intro}
Temporal probabilistic relational models express relations between objects, modelling uncertainty as well as temporal aspects.
Within one time step, a temporal model is considered static.
When time advances, the current model state transitions to a new state.
Performing inference on such models requires algorithms to efficiently handle the temporal aspect to be able to efficiently answer queries.

Reasoning in lifted representations has a complexity polynomial in domain sizes.
But, models dissolve into ground instances through evidence, which no longer permits reasoning in polynomial time, making query answering infeasible for any reasoning algorithm, exact or approximate.
Thus, a key challenge during inference in temporal models is to restore a lifted, i.e., non-grounded, representation.
Therefore, we formulate and study the problem of \ac{krp} in temporal models to tame the effect of evidence for efficient query answering. 

First-order probabilistic inference leverages the relational aspect of a static model, using representatives for groups of indistinguishable, known objects, also known as lifting \cite{poole2003first}. 
\citeauthor{poole2003first} \shortcite{poole2003first} presents parametric factor graphs as relational models and proposes \ac{lve} as an exact inference algorithm on relational models.
\citeauthor{TagFiDaBl13} \shortcite{TagFiDaBl13} extend \ac{lve} to its current form.
%\citeauthor{lauritzen1988local} \shortcite{lauritzen1988local} introduce the junction tree algorithm.
To benefit from the ideas of the junction tree algorithm \cite{lauritzen1988local} and \ac{lve}, \citeauthor{BrMoe16a} \shortcite{BrMoe16a} present the \ac{ljt} for exact inference given a set of queries.
To answer multiple temporal queries, \citeauthor{gehrke2018ldjt} \shortcite{gehrke2018ldjt} present the \ac{ldjt}, which combines the advantages of the interface algorithm~\cite{Murphy:2002:DBN} and \ac{ljt}.
Other approaches for temporal relational models perform approximate inference.
\citeauthor{ahmadi2013exploiting} \shortcite{ahmadi2013exploiting} propose a colour passing scheme to obtain a lifted representation of a \ac{dmln} using exact symmetries and extend lifted belief propagation \cite{singla2008lifted} for temporal approximate inference.
Further inference algorithms for \acp{dmln} exist \cite{geier2011approximate,papai2012slice}.
But, to the best of our knowledge, none of these approaches tackle the \ac{krp} problem.

For static relational models, approaches exist to approximate symmetries as evidence may ground even a static model \cite{van2012conditioning}.
\citeauthor{van2013complexity} \shortcite{van2013complexity} approximate lifted binary evidence.
\citeauthor{singla2014approximate} \shortcite{singla2014approximate} propose approximate lifting techniques, which group together distinguishable objects and treat them identically.
\citeauthor{venugopal2014evidence} \shortcite{venugopal2014evidence} form clusters of objects and project the marginal distribution of one object to all objects of a cluster.
Both approaches introduce an unknown bias into the distributions of the groups. 
\citeauthor{van2015lifted} \shortcite{van2015lifted} present an unbiased approach for approximating symmetries. 
However, these approaches do not account for temporal aspects.

Thus, we present \ac{tam} as an approach to solve the \ac{krp} problem in temporal models.
Specifically, \ac{tam} incorporates 
\begin{enumerate*}
	\item clustering to group submodels and 
	\item statistical significance checks to test the groups to be merged.
\end{enumerate*} 
Model structure and behaviour are captured in a set of functions that define local distributions for the \acp{rv} in the model.
Clustering forms groups of functions based on the similarity between local distributions.
The significance checks allow for determining the fitness of the clustering outcome.
If the clustering is deemed fit, each group is merged, yielding an unbiased approximation.
In exchange for faster runtime, \ac{tam} introduces a bounded error, which becomes negligible over time.

\citeauthor{boyen1998tractable} \shortcite{boyen1998tractable} show that for stationary processes, evidence can lead to conditional dependences in temporal probabilistic propositional models, making inference infeasible.
They propose to introduce additional \acp{rv} to achieve conditional independences between subprocesses even under evidence. 
Further, \citeauthor{boyen1998tractable} show that, for any approximation scheme of belief state representations, the error contracts exponentially as the process evolves, making the introduced error bounded indefinitely \cite{boyen1998tractable}.
Their approach and \ac{tam} are related as in both cases evidence makes inference infeasible.
However, \ac{tam} aims at \emph{automatically} restoring a lifted representation.
In summary, the cause, namely evidence, is the same for both problems but the problems are different and the means to make inference possible again differ highly.

\ac{tam} is applicable in different formalisms and algorithms.
However, we discuss \ac{tam} as part of \ac{ldjt} for two reasons:
First, when advancing in time, \ac{ldjt} computes a minimal message that is the source of the most splits of the next time step.
Applying \ac{tam} on this message tackles the \ac{krp} problem at its root.
Second, using \ac{tam} with an exact algorithm allows for attributing errors to merging rather than imprecisions during reasoning.
Additionally, \ac{tam} is deterministic in its approximation, thereby, avoiding problems with sampling rates or ergodicity.
Empirical results show that \ac{tam} significantly improves performances of \ac{ldjt}, while keeping errors small and attributable to merging.

\begin{figure*}
    \centering
	\begin{minipage}[b]{.28\textwidth}
        \centering
        \scalebox{0.86}{
        \begin{tikzpicture}[rv/.style={draw, ellipse, inner sep=1pt},pf/.style={draw, rectangle, fill=gray},label distance=0.2mm]
        	\node[rv, inner sep=0.5pt]					 								(S)	{$R(X)$ };
            \pfs{left}{S}{20mm}{230}{$g^0$}{USa}{US}{USb}    
        	\node[rv, left of=US, node distance=20mm, inner sep=0.5pt]			(U)	{$Pub(X,J)$};    
        	\node[rv, below of=S, node distance=10mm, inner sep=0.5pt]			(T1)	{$A(X)$};    
            \pfs{left}{T1}{20mm}{230}{$g^1$}{ASa}{AS}{ASb}    
        	\node[rv, left of=AS, node distance=20mm, inner sep=0.5pt]			(A)	{$D(X)$};    
    
        	\draw (U) -- (US);
        	\draw (US) -- (S);
        	\draw (US) -- (T1);

        	\draw (A) -- (AS);
        	\draw (AS) -- (S);
        	\draw (AS) -- (T1);

        \end{tikzpicture}
        }
        \caption{Parfactor graph for $G^{ex}$}
        \label{fig:swe}	% label to refer figure in text
    \end{minipage}\hfill
    \begin{minipage}[b]{.69\textwidth}
\centering
\scalebox{0.86}{
%\tikzstyle{every node}=[font=\Large]
\begin{tikzpicture}[rv/.style={draw, ellipse, inner sep=1pt},pf/.style={draw, rectangle, fill=gray},label distance=0.2mm]
    % Place nodes
	\node[rv]					 								(S)	{$\mathbf{R_{t-1}(X)}$};
    \pfs{left}{S}{20mm}{230}{$g^0_{t-1}$}{USa}{US}{USb}    
	\node[rv, left of=US, node distance=20mm, inner sep=0.5pt]			(U)	{$Pub_{t-1}(X,J)$};    
	\node[rv, below of=S, node distance=10mm]						(T1)	{$\mathbf{A_{t-1}(X)}$};    
    \pfs{left}{T1}{20mm}{230}{$g^1_{t-1}$}{ASa}{AS}{ASb}    
	\node[rv, left of=AS, node distance=20mm, inner sep=0.5pt]			(A)	{$D_{t-1}(X)$};    
    
	\node[rv, right of = S, node distance=80mm]					 		(S1)	{$R_{t}(X)$};
    \pfs{left}{S1}{20mm}{230}{$g^0_{t}$}{USa}{US1}{USb}    
	\node[rv, left of=US1, node distance=20mm, inner sep=0.5pt]			(U1)	{$Pub_{t}(X,J)$};    
	\node[rv, below of=S1, node distance=10mm]						(T11)	{$A_{t}(X)$};    
    \pfs{left}{T11}{20mm}{230}{$g^1_{t}$}{ASa}{AS1}{ASb}    
	\node[rv, left of=AS1, node distance=20mm, inner sep=0.5pt]			(A1)	{$D_{t}(X)$};

    %\node[right of = S, node distance = 1cm] (I) {};
    
    \pfs{right}{S}{18mm}{315}{$g^R$}{UAa}{IU}{UAb}

    \path [-] (IU) edge node {} (T1);
    \path [-, bend left=15] (IU) edge node {} (S1);
	\draw (IU) -- (S);
    
	\draw (U) -- (US);
	\draw (US) -- (S);
	\draw (US) -- (T1);

	\draw (A) -- (AS);
	\draw (AS) -- (S);
	\draw (AS) -- (T1);

	\draw (U1) -- (US1);
	\draw (US1) -- (S1);
	\draw (US1) -- (T11);

	\draw (A1) -- (AS1);
	\draw (AS1) -- (S1);
	\draw (AS1) -- (T11);

    %\draw[->] (2.3,0)node[right] {start} -- (1.7, 0);

\end{tikzpicture}
}
\caption{$G_\rightarrow^{ex}$ the two-slice temporal parfactor graph for model $G^{ex}$}
\label{fig:TSPG}	% label to refer figure in text 
\end{minipage}
\end{figure*}

In the following, we recapitulate \acp{pdm} as a formalism for specifying temporal probabilistic relational models and \ac{ldjt} for efficient query answering in \acp{pdm}. %, an efficient reasoning algorithm for \acp{pdm}.
Then, we present \ac{tam}, which includes clustering, significance checks, and merging. %, and error identification.
Lastly, we evaluate \ac{tam} theoretically and empirically. %and %, which includes \emph{smoothing} with message passing. 
% conclude by looking at extensions.

\section{Preliminaries}
We shortly present \acp{pm} \cite{braun2018parameterised}, then extend \acp{pm} to the temporal case, resulting in \acp{pdm}, and recapitulate \ac{ldjt} \cite{gehrke2018ldjt,GehBrMo19a}, a \emph{smoothing}, \emph{filtering}, and \emph{prediction} algorithm for \acp{pdm}.

\subsection{Parameterised Probabilistic Models}\label{pm}
\acp{pm} combine first-order logic with probabilistic models, using \acp{lv} as parameters to represent sets of indistinguishable constructs.
As an example, we set up a \ac{pm} to model the reputation of researchers, inspired by the competing workshop example \cite{milch2008lifted}, with a \ac{lv} representing researchers.
A reputation is influenced by activities such as publishing, doing active research, and attending conferences.
A \ac{rv} parameterised with \acp{lv} forms a \ac{prv}.

\begin{definition}
	Let $\mathbf{R}$ be a set of randvar names, $\mathbf{L}$ a set of logvar names, $\Phi$ a set of factor names, and $\mathbf{D}$ a set of constants (universe).
	All sets are finite.
	Each logvar $L$ has a domain $\mathcal{D}(L) \subseteq \mathbf{D}$.
	A \emph{constraint} is a tuple $(\mathcal{X}, C_{\mathbf{X}})$ of a sequence of logvars $\mathcal{X} = (X^1, \dots, X^n)$ and a set $C_{\mathcal{X}} \subseteq \times_{i = 1}^n\mathcal{D}(X_i)$.
	The symbol $\top$ for $C$ marks that no restrictions apply, i.e., $C_{\mathcal{X}} = \times_{i = 1}^n\mathcal{D}(X_i)$.
	A \emph{PRV} $R(L^1, \dots, L^n), n \geq 0$ is a syntactical construct of a randvar $R \in \mathbf{R}$ possibly combined with logvars $L^1, \dots, L^n \in \mathbf{L}$. % to represent a set of randvars.
	If $n = 0$, the PRV is parameterless and forms a propositional randvar.
	A PRV $A$ or logvar $L$ under constraint $C$ is given by $A_{|C}$ or $L_{|C}$, respectively.
	We may omit $|\top$ in $A_{|\top}$ or $L_{|\top}$.
	The term $\mathcal{R}(A)$ denotes the possible values (range) of a PRV $A$. 
	An \emph{event} $A = a$ denotes the occurrence of PRV $A$ with range value $a \in \mathcal{R}(A)$.
\end{definition}

We use the \ac{rv} names $A$, $D$, $R$, and $Pub$ for attends conference, does research, reputation, and publishes in journals, respectively, and $\mathbf{L} = \{X, J\}$ with $\mathcal{D}(X) = \{x_1, x_2, x_3\}$ (people) and $\mathcal{D}(J) = \{j_1, j_2\}$ (journals).
We build boolean \acp{prv} $A(X)$, $D(X)$, $R(X)$, and $Pub(X,J)$.
A \ac{pf} describes a function, mapping argument values to real values (potentials). %, of which at least one is non-zero.

\begin{definition}
	We denote a \emph{parfactor} $g$ by $\phi(\mathcal{A})_{| C}$ with $\mathcal{A} = (A^1, \dots, A^n)$ a sequence of PRVs, $\phi : \times_{i = 1}^n \mathcal{R}(A^i) \mapsto \mathbb{R}^+$ a function with name $\phi \in \Phi$, and $C$ a constraint on the logvars of $\mathcal{A}$. %, identical for all possible groundings of $\mathcal{A}$ w.r.t.\ $C$
	We may omit $|\top$ in $\phi(\mathcal{A})_{| \top}$.
	The term $lv(Y)$ refers to the \acp{lv} in some element $Y$, a PRV, a parfactor or sets thereof.
	The term $gr(Y_{| C})$ denotes the set of all instances of $Y$ w.r.t.\ constraint $C$.
	A set of parfactors forms a \emph{model} $G := \{g^i\}_{i=1}^n$.
	The semantics of $G$ is given by grounding and building a full joint distribution.
	With $Z$ as the normalisation constant, $G$ represents $P_G = \frac{1}{Z} \prod_{f \in gr(G)} f$.
\end{definition}

Let us build the \ac{pm}
$G_{ex}$$=$$\{g^i\}^1_{i=0}$, shown in \cref{fig:swe}, with
 $g^0 = \phi^0(R(X), A(X),  Pub(X,J))_{| \top}$ and
 $g^1 = \phi^1(R(X), \linebreak A(X),  D(X))_{| \top}$, each with eight input-output pairs (omitted).
 Next, we present a temporal extension of a \ac{pm}.

\subsection{Parameterised Probabilistic Dynamic Models}\label{sec:pdm}

We define \acp{pdm} based on the first-order Markov assumption.
Further, the underlying process is stationary. 
\begin{definition}
	A \ac{pdm} $G$ is a pair of \acp{pm} $(G_0,G_\rightarrow)$ where
        $G_0$ is a PM representing the first time step and  
        $G_\rightarrow$ is a \acl{2tpm} representing $\mathbf{A}_{t-1}$ and $\mathbf{A}_t$ where $\mathbf{A}_\pi$ a set of \acp{prv} from time slice $\pi$.
        The semantics of $G$ is to instantiate $G$ for a given number of time steps, resulting in a \ac{pm} as defined above.
\end{definition}

\Cref{fig:TSPG} shows 
$G_\rightarrow^{ex}$ consisting of $G^{ex}$ for time slice $t-1$ and $t$ with \emph{inter}-slice \acp{pf} for the behaviour over time. 
The \ac{pf} $g^{R}$ is the \emph{inter}-slice \ac{pf}. 
For example, we can observe AAAI conference attendance, which changes over time as, unfortunately, getting papers accepted at consecutive conferences is difficult.
Nonetheless, people with high attendance usually have a good reputation.

In general, a query asks for a probability distribution of a \ac{rv} given fixed events as evidence.

\begin{definition}\label{def:query_dynamic}
    Given a \ac{pdm} $G$, a query term $Q$ (ground \ac{prv}), and events $\mathbf{E}_{0:t} = \{E^i_t=e^i_t\}_{i,t}$, the expression $P(Q_t|\mathbf{E}_{0:t})$ denotes a \emph{query} w.r.t.\ $P_G$.
\end{definition}
The problem of answering a query $P(A^i_\pi|\mathbf{E}_{0:t})$ w.r.t.\ the model is called  \emph{filtering} for $\pi = t$ and \emph{prediction} for $\pi > t$.
In this paper, we focus on such temporal queries. 

\subsection{Query Answering Algorithm: LDJT}\label{sec:fodjt}

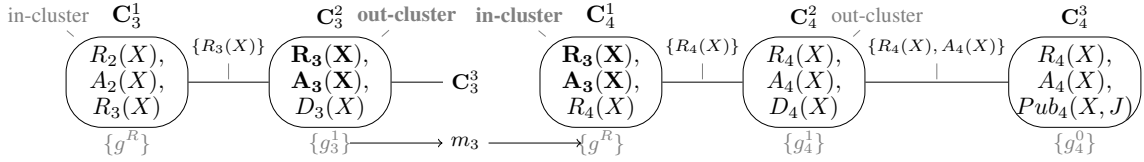
\begin{figure*}[t]
    \centering
%\center
\scalebox{0.9}{
\begin{tikzpicture}[every node/.style={font=\footnotesize}, node distance=30mm]
	\node[pc, label={[gray, inner sep=1pt]270:{$\{g^R\}$}},    pin={[pin distance=1mm, gray, align=center]120:{in-cluster}}, label={[font=]90:{$\mathbf{C}^1_3$}}]				(c1) {$R_{2}(X),$\\ $A_{2}(X),$ \\$R_{3}(X)$};
    
	\node[pc, right of=c1,     pin={[pin distance=1mm, gray, align=center]70:{\textbf{out-cluster}}},
label={[gray, inner sep=1pt]270:{$\{g^1_3 \}$}},label={90:{$\mathbf{C}^2_3$}}]	(c2) {$\mathbf{R_{3}(X)},$\\$ \mathbf{A_{3}(X)},$ \\ $D_{3}(X)$};
    %\node[pc, below of=c2, node distance=2.5cm, label={[gray, inner sep=1pt]270:{$\{g^1_3\}$}},label={90:{$\mathbf{C}^3_3$}}]	(c3) {$Hot_{3},$\\$ AttC_{3}(A),$ \\ $DoR_{3}(X)$};

    \node[below of = c2, node distance=0.9cm, xshift= 0.17cm] (a) {};
    
    \node[right of = c2, node distance=2cm] (c) {$\mathbf{C}^3_3$};
    \node[below of = c, node distance=0.9cm] (c11) {$m_{3}$};
    %!\node[above of = c, node distance=0.8cm] (b) {$\beta_4$};
    
	\node[pc, right of=c, node distance=2cm, label={[gray, inner sep=1pt]270:{$\{g^R\}$}},  pin={[pin distance=1mm, gray, align=center]120:{\textbf{in-cluster}}}, label={[font=]90:{$\mathbf{C}^1_4$}}]				(c4) {$\mathbf{R_{3}(X)},$\\ $\mathbf{A{_3}(X)},$ \\$R_{4}(X)$};
	\node[pc, right of=c4,     pin={[pin distance=1mm, gray, align=center]70:{out-cluster}},
label={[gray, inner sep=1pt]270:{$\{g^1_4\}$}},label={90:{$\mathbf{C}^2_4$}}]	(c5) {$R_{4}(X),$\\$ A_{4}(X),$ \\ $D_{4}(X)$};
    \node[pc, right of=c5, node distance = 40mm, label={[gray, inner sep=1pt]270:{$\{g^0_4\}$}},label={90:{$\mathbf{C}^3_4$}}]	(c6) {$R_{4}(X),$\\$ A_{4}(X),$ \\ $Pub_{4}(X,J)$};
      
    \node[below of = c4, node distance=0.9cm, xshift= -0.17cm] (b) {};

	\draw (c1) -- node[inner sep=1pt, pin={[yshift=-2.4mm]90:{\scriptsize$\{R_3(X)\}$}}]		{} (c2);
	\draw (c2) -- node[]		{} (c);

	\draw (c4) -- node[inner sep=1pt, pin={[yshift=-2.4mm]90:{\scriptsize$\{R_4(X)\}$}}]		{} (c5);
	\draw (c5) -- node[inner sep=1pt, pin={[yshift=-2.4mm]90:{\scriptsize$\{R_4(X), A_4(X)\}$}}]		{} (c6);

    \path [<-] (c11) edge node [] {} (a);
    \path [->] (c11) edge node [] {} (b);
    %\path [-] (b) edge node [label= above:{$\cup$}] {} (c2.east);
    %\path [-] (b) edge node [label= above:{$\sum$}] {} (c4.west);

\end{tikzpicture}
}
\caption{FO jtree $J_3$ without $\mathbf{C}^3_3$ and FO jtree $J_4$ connected with $m_{3}$}
\label{fig:fojt1}	% label to refer figure in text 
\end{figure*}

The important property of \ac{ldjt} \cite{gehrke2018ldjt} for this paper is that \ac{ldjt} constructs \acp{fojt} to efficiently answer multiple queries using \ac{lve}.
The \acp{fojt} in \ac{ldjt} contain a minimal set of \acp{prv} to m-separate time steps, which means that information about these \acp{prv} renders \acp{fojt} independent from each other.
Let us now define an \ac{fojt}, with parameterised clusters (parclusters) as nodes, and present how \ac{ldjt} proceeds in time. 

\begin{definition}
	Let $\mathbf{X}$ be a set of logvars, $\mathbf{A}$ a set of PRVs with $lv(\mathbf{A}) \subseteq \mathbf{X}$, and $C$ a constraint on $\mathbf{X}$.
	Then, $\mathbf{A}_{|C}$ denotes a \emph{parcluster}.
	We omit ${|C}$ if $C = \top$ and $lv(\mathbf{A}) = \mathbf{X}$.
	An \emph{FO jtree} for a model $G$ is a cycle-free graph $J = (V, E)$, where $V$ is the set of nodes, i.e., parclusters, and $E$ the set of edges.
	$J$ must satisfy three properties:
	\begin{enumerate*}
		\item A parcluster $\mathbf{C}^i$ is a set of \acp{prv} from $G$.
		\item For each \ac{pf} $\phi(\mathcal{A})_{ | C}$  in G, $\mathcal{A}$ must appear in some parcluster $\mathbf{C}^i$.
		\item If a \ac{prv} from $G$ appears in two parclusters $\mathbf{C}^i$ and $\mathbf{C}^j$, it must also appear in every parcluster $\mathbf{C}^k$ on the path connecting nodes $i$ and $j$ in $J$.
	\end{enumerate*}
	The parameterised set $\mathbf{S}^{ij}$, called \emph{separator} of edge $\{i, j\} \in E$, is defined by $\mathbf{C}^i \cap \mathbf{C}^j$.
	Each $\mathbf{C}^i \in V$ has a \emph{local model} $G^i$ and $\forall g \in G^i$: $rv(g) \subseteq \mathbf{C}^i$.
	The $G_i$'s partition $G$.
\end{definition}

Querying a minimal set of \acp{prv} with \ac{lve} in an \ac{fojt} combines all information to m-separate time steps.
To obtain the minimal set, \ac{ldjt} uses interface \acp{prv} $\mathbf{I}_t$ of $G_\rightarrow$. 

\begin{definition}
    The forward interface $\mathbf{I}_{t-1}$ is given by \[\mathbf{I}_{t-1} = \{A_{t}^i \mid \exists \phi(\mathcal{A})_{ | C} \in G :  A_{t-1}^i \in \mathcal{A} \wedge \exists A_{t}^j \in \mathcal{A}\}.\]
\end{definition}

\acp{prv} $R_{t-1}(X)$ and $A_{t-1}(X)$ from $G_{\rightarrow}^{ex}$, shown in \cref{fig:TSPG}, make up $\mathbf{I}_{t-1}$.
While constructing \ac{fojt} structures, \ac{ldjt} ensures that the \ac{fojt} $J_t$ for time step $t$ has a parcluster containing $\mathbf{I}_{t-1}$, which is called \emph{in-cluster}, and a parcluster containing $\mathbf{I}_{t}$, which is called \emph{out-cluster},.
The \emph{in-} and \emph{out-clusters} allow for reusing the \ac{fojt} structures.

To proceed in time, \ac{ldjt} calculates a forward message $m_t$ over $\mathbf{I}_t$ using the \emph{out-cluster} of $J_t$. 
Hence, $m_t$ contains exactly the necessary information, as a set of \acp{pf}, to be able to answer queries in the next time step.
Afterwards, \ac{ldjt} adds $m_{t}$ to the local model of the \emph{in-cluster} of $J_{t+1}$.

\Cref{fig:fojt1} depicts passing on the current state from time step $3$ to $4$.
To capture the state at $t=3$, \ac{ldjt} sums out the non-interface \ac{prv} $D(X)$ from the local model and received messages of $\mathbf{C}_3^2$ and saves the result in message $m_{3}$.
Increasing $t$ by one, \ac{ldjt} adds $m_{3}$ to $\mathbf{C}_4^1$'s local model. 

\section{Temporal Approximate Merging}
In a temporal probabilistic relational model, evidence can slowly ground the model over time by introducing splits. 
We propose to name the problem of \emph{finding how to undo splits to retain a lifted solution over time, keeping any error unbiased and acceptable}, as the \ac{krp} problem.
Retaining a lifted solution over time means that lifted algorithms run in polynomial time w.r.t.\ the domain size if a lifted solution exists \cite{niepert2014tractability}.
To solve the \ac{krp} problem, an approach is required to identify any number of clusters based on how similar $\phi$'s of \acp{pf} are and combine them.
To keep the error unbiased and acceptable, the groundings need to be accounted for and the identified cluster means have to discriminate the clusters.
Unfortunately, to combine similar $\phi$'s, we cannot use the colouring algorithm \cite{ahmadi2013exploiting} as it uses exact symmetries.
Before presenting \ac{tam}, let us formulate the problem in terms of \acp{pdm}.

Even though \ac{ldjt} instantiates vanilla \ac{fojt} structures, $m_t$ carries over splits caused by evidence.
Formally, the problem is that in a model $G_t = \{g_t^i\}_{i=1}^n$ at time step $t$, many parfactors are split.
Whenever evidence leads to a split of a parfactor, the split carries over to subsequent time steps.
Thus, $G_t$ has the following form:
\begin{align}
	\{g_t^{i,1}, \dots, g_t^{i,m}\}_{i=1}^{n}, m \in \mathbb{N}^+. \label{eq:input}
\end{align}
For each $i$, the different $g_t^{i,j} = \phi_t^{i,j}(\mathcal{A}^i)_{|C^{i,j}}$, $1 \leq j \leq m$, have the same arguments $\mathcal{A}^i$ but different constraints $C^{i,j}$ and varying functions $\phi_t^{i,j}$ as a result of evidence.
The assumption is that some $g_t^{i,j}$ have similar $\phi$'s as differences introduced by evidence are minimal or otherwise are overcome by model behaviour over time, i.e., potentials align again.
Then, one can combine similar $\phi$'s while introducing only a small and bounded error in exchange for faster reasoning. 
In the following, we show that the assumption holds, by showing that $\phi$'s converge, allowing them to be merged, and that the error \ac{tam} introduces is bounded.

The idea for restoring a lifted representation is to merge those $g_t^{i,j}$ with similar $\phi$'s into one parfactor
\begin{align}
	g_t^{i,k} = \phi_t^{i,k}(\mathcal{A}^i)_{|C^{i,k}} \label{eq:merged}
\end{align}
where $\phi_t^{i,k}$ represents a merged version of the combined $\phi_t^{i,j}$ and $C^{i,k}$ is a union of the combined $C^{i,j}$.
Merging all parfactors that behave similarly for each $i$ leads to a $G_t'$ of the following form with parfactors as in \cref{eq:merged} and $l < m$:
\begin{align}
	\{g_t^{i,1}, \dots, g_t^{i,l}\}_{i=1}^{n} \label{eq:output}
\end{align}

With \ac{tam}, we present a merging scheme that takes a model $G$ as given in \cref{eq:input} and computes a model $G'$ as given in \cref{eq:output}.
It is reasonable to apply \ac{tam} to $G$ when transitioning from time step $t$ to $t+1$ as the transition transfers any splits as well.
In general, $G$ may be any parfactor model and one may also transfer the idea to a \ac{dmln} model \cite{ahmadi2013exploiting}. %after making the functions in \ac{tam} apply to \acp{dmln}.
But, models may be very large, e.g., the union of all local models of an FO jtree $J_t$, such that finding groups for each $i$ is too costly.
%Merging parfactors before starting with $t+1$ reduces the number of splits at $t+1$.
Therefore, we propose to make \ac{tam} a subroutine of \ac{ldjt}.
Transitioning from $t$ to $t+1$ requires computing message $m_t$, which provides a state description of $t$ that is relevant to $t+1$.
%describes the current state at $t$.
Applying \ac{tam} to $m_t$ prepares a message with fewer groups, leading to fewer splits in $J_{t+1}$.
Additionally, $m_t$ normally has considerably fewer parfactors than $G_t$.
Next, we explain in detail how to get from \cref{eq:input} to \cref{eq:output} with \ac{tam}.

\subsection{Keeping Reasoning Polynomial with \ac{tam}}
\Cref{alg:tam} outlines \ac{tam} to solve the \ac{krp} problem.
Inputs are a model $G$, possibly $m_t$, as well as two additional parameters, radius $\epsilon$ and significance level $\alpha$, which become important later on.
The first step is to preprocess $G$ for easier handling in subsequent steps.
The main loop describes how a clustering algorithm identifies groups for merging and how groups are merged if \ac{tam} deems the clusters to fit.
The upcoming paragraphs discuss the individual steps of \cref{alg:tam}.

\begin{algorithm}[b]
	\caption{Temporal Approximate Merging}
	\label{alg:tam}
	\begin{algorithmic}
		\Procedure{TAMe}{Model $G$, Radius $\epsilon$, Significance $\alpha$}
%			\For {each partition $P \in \mathbf{P}$}
%				\State $P \gets$ multiply overlapping parfactors			\Comment \cref{eq:mult}
%			\EndFor
			\State $\mathbf{P} \gets$ partitioning of $G$ based on \acp{lv} 		\Comment \cref{eq:partition}
%			\State Multiply overlapping parfactors in each $P \in \mathbf{P}$ 	%\Comment  \cref{eq:mult}
			\For {each partition $P \in \mathbf{P}$}
				\State $P \gets$ multiply overlapping parfactors				\Comment \cref{eq:mult}
				\State $\mathbf{K} \gets$ \Call{DBSCAN}{$P$, $\epsilon$, $2$, $rsim$}  
				\If{\Call{ANOVA}{$\mathbf{K}$, $rsim'$, $\alpha$} rejects $H_0$}
					\State $G \gets G \setminus P$
					\For {each cluster $K \in \mathbf{K}$}
						\State $G \gets G \cup\{K \text{ merged}\}$		\Comment \cref{eq:mean}
					\EndFor
				\EndIf
			\EndFor
		\EndProcedure
	\end{algorithmic}  
\end{algorithm}

\subsubsection{Model Partitioning}
The preprocessing of $G$ is a consequence of the following considerations.
A challenge that arises from a model as in \cref{eq:input} is that merging parfactors for each $i$ independent of each other may lead to different groups that cause splits again, undoing any merging efforts.
Using an $i$ at random and transferring the grouping of the $i$ parfactors to all other parfactors may lead to unreasonable groups for the other $i$'s.
A safe option is to multiply parfactors with overlapping constraints into one parfactor which in a worst case leads to $n=1$ and very large parfactors that no longer explicitly represent independencies and may complicate calculations for messages and queries.
Within \ac{ldjt}, one could trace back if a set of parfactors in $m_t$ originates from the message that has come from the direction of the \emph{in-cluster} to the \emph{out-cluster} as this message contains information about the past and is the origin of the most splits in $m_t$.
Therefore, it may be possible to identify a unique $i$ in $m_t$ as a reasonable source for merging. 
However, there are no guarantees to find such an $i$.
Instead, we opt to partition the parfactors in $G$ based on the logvars appearing in $G$ into a set $\mathbf{P}$ of sets of parfactors.
Each partition $P \in \mathbf{P}$ has a set of logvars $\mathbf{X}_p$ that has been affected in the same way by splitting due to evidence.
Formally, $P$ has the form 
\begin{align}
	P = \{g_t^{i,1}, \dots, g_t^{i,m}\}_{i=1}^{n_p} \label{eq:partition}
\end{align}
with $lv(g_t^{i,j}) \subseteq \mathbf{X}_p$.
The next step is to identify groups of parfactors in each partition that behave similarly.

\subsubsection{Parfactor Clustering}
After partitioning $G$, each partition $P \in \mathbf{P}$ of the form in \cref{eq:partition} has parfactors whose constraints overlap between all $i$ for each $j$.
Therefore, \ac{tam} multiplies all parfactors with overlapping constraints into one parfactor before starting with identifying groups.
If $lv(g_t^{i,j}) = \mathbf{X}_p$, each $i$ refers to $m$ parfactors with the same constraint over all $i$'s for each $j$, i.e, the constraints are the same at position $j$ for all $i$'s.
Then, multiplication in $P$ to combine \acp{prv} with the same constraints boils down to
\begin{align}
	P = \left\{\prod_{i = 1}^{n_p} g_t^{i,1}, \dots, \prod_{i = 1}^{n_p} g_t^{i,m}\right\} = \{g_t^{p,1}, \dots, g_t^{p,m} \} \label{eq:mult}
\end{align}
where multiplying parfactors corresponds to the LVE operation of \emph{multiply}, c.f.\ \cite{TagFiDaBl13}.

To identify groups of parfactors with similar behaviour, one needs to specify 
\begin{enumerate*}
	\item what ``similar behaviour'' means and 
	\item how to find such groups automatically.
\end{enumerate*}
We first consider the second item, which influences specifying the first item.

\ac{tam} needs to identify an unknown number of groups based on how similar $\phi$'s are.
Density-based clustering groups similar points into an unknown number of groups. %, which \ac{tam} needs to do.
Therefore, \ac{tam} uses density-based clustering.
For the evaluation, we instantiate \ac{tam} with DBSCAN \cite{ester1996density,schubert2017dbscan} as the clustering approach.
In the following, we illustrate how density-based clustering fits into the overall scheme of \ac{tam} using  DBSCAN.
DBSCAN identifies data points as core points if in their neighbourhoods, determined by a radius $\epsilon$ around a point, lie a certain number $minPts$ of other data points.
A core data point makes up a cluster along with all the data points in its neighbourhood, which recursively proceeds with the next core data point in the neighbourhood.
To determine data points in a neighbourhood, DBSCAN requires a distance function as an input.
DBSCAN is able to detect outliers, which do not occur in any neighbourhood.
For the purpose of clustering parfactors, we set $minPts$ to $2$ to be able to cluster even two parfactors.
The distance measure should assess how similarly parfactors behave, with $0$ meaning identical behaviour and larger values meaning less similar behaviour.

To determine the similarity of the behaviour of two parfactors, one could calculate marginal distributions for a \ac{prv} that occurs with split constraints and compare if the marginals are in a certain $\delta$ area.
However, marginal distributions could result from completely different potentials and be similar by chance.
The potentials of a \ac{pf} on the other hand specify the current weight for each possible assignment.
Thus, in case the ratio of the potentials of two \acp{pf} are similar, they also have similar marginal distributions and behave similarly.
For example, a \ac{pf} mapping to $4$ and $2$ and another \ac{pf} mapping to $8.1$ and $3.9$ behave similarly.
Both \acp{pf} weight the first assignment about twice as much as the second.
Assuming both \acp{pf} are independent from the rest and only have one grounding each, the marginals for $true$ would be $0.667$ and $0.675$ respectively, i.e., less than $0.01$ apart from each other.
The same case arises for two parfactors mapping to $\langle 4, 2\rangle$ and $\langle 4.1, 1.9 \rangle$ respectively.

Such potentials, when thought of as vectors, have a small angle between them, i.e., a high cosine similarity, which we use to specify ``similar behaviour''.
For the setup of the similarity of two parfactors $g_t^{i,j_1} = \phi_t^{i,j_1}(\mathcal{A}^i)_{|C^{i,j_1}}$ and $g_t^{i,j_2} = \phi_t^{i,j_2}(\mathcal{A}^i)_{|C^{i,j_2}}$, we use a function $rsim : (\times_{i = 1}^n range(A^i) \mapsto \mathbb{R}^+,\times_{i = 1}^n range(A^i) \mapsto \mathbb{R}^+) \mapsto \mathbb{R}^+$ that is defined as follows:
\begin{flalign}
	&rsim(\phi_t^{i,j_1}, \phi_t^{i,j_2}) = \nonumber\\ 
	&1 - \frac{\displaystyle\sum_{\mathbf{a} \in range(\mathcal{A}^i)}\phi_t^{i,j_1}(\mathbf{a})\cdot \phi_t^{i,j_2}(\mathbf{a})}{\sqrt{\displaystyle\sum_{\mathbf{a} \in range(\mathcal{A}^i)}\phi_t^{i,j_1}(\mathbf{a})^2}\cdot \sqrt{\displaystyle\sum_{\mathbf{a} \in range(\mathcal{A}^i)}\phi_t^{i,j_2}(\mathbf{a})^2}} \label{eq:rsim}
\end{flalign}
The result of \cref{eq:rsim} lies in the interval $[0,1]$.
We calculate $1$ minus the fraction to get a ``distance'' measure, in which a lower value means a closer distance.

As a consequence of $rsim$ with its codomain $[0,1]$ as the distance function for DBSCAN, $\epsilon$ needs to be $\leq 1$.
Overall, the inputs of DBSCAN for clustering parfactors are a partition $P$ of parfactors, $\epsilon$, $minPts = 2$, and $rsim$.
$\epsilon$ trades off cluster sizes with accuracy.
The output is a clustering (partitioning) of $P$, i.e., a set $\mathbf{K}$ of sets in which each $K \in \mathbf{K}$ is a set of parfactors that are assumed to behave similarly.

\subsubsection{Fitness of Clustering}
The question that remains after clustering is: How good is the clustering?
The clustering is highly influenced by the choice of $\epsilon$, which leads to large clusters if set to a high value but may also blur the potentials in the merged parfactor to a higher degree.

One could calculate the error introduced by the clustering w.r.t.\ a given \ac{prv} $A$ by comparing marginal distributions of $A$ before and after merging.
However, if a model already is highly shattered, the computational effort can be very high to compute marginal distributions before merging.

DBSCAN clusters together parfactors with a small angle between them.
So a clustering fits if the variance of angles within clusters is low and the variance of angles between clusters is high.
Analysis of variance (ANOVA) \cite{fisher1925statistical} is a statistical method to test for significance of a clustering.
In our setup, ANOVA computes the variance of each parfactor in a cluster $K \in \mathbf{K}$ to the mean parfactor of $K$ as well as the variance of the mean parfactor of $K$ to the mean parfactor of all points in $\mathbf{K}$.
Hence, it provides an indication of how good the clustering separates parfactors.

ANOVA is used to accept or reject hypotheses.
The default hypothesis is that the means of all clusters are equal.
For our problem, the default hypothesis $H_0$ is that the mean parfactors of the clusters are equal, i.e., are not statistically significant to discriminate clusters.
The goal is to be able to reject $H_0$, that is to say there is more difference between than within clusters.
In case \ac{tam} can reject $H_0$, at least one cluster is significantly different from the others.

To compute a mean parfactor of a cluster $K$, \ac{tam} calculates the average of all potentials while accounting for groundings.
Formally, given a set of parfactors $\{\phi_t^{i,j}(\mathcal{A}^i)_{|C^{i,j}}\}_{j=1}^m$, a mean parfactor $g_t^{i,k} = \phi_t^{i,k}(\mathcal{A}^i)_{|C^{i,k}}$ is determined by
\begin{align}
	\phi_t^{i,k}(\mathbf{a}) = \frac{\sum_{j=1}^{m}gr(\phi_t^{i,j}(\mathbf{a})_{|C^{i,j}})\phi_t^{i,j}(\mathbf{a})}{gr(\phi_t^{i,k}(\mathbf{a})_{|C^{i,k}})} \label{eq:mean}
\end{align}
for each $\mathbf{a} \in range(\mathcal{A}^i)$ and $C^{i,k}$ is a union of the different $C^{i,j}$.
Thus, \ac{tam} goes through all potentials and for each assignment, adds the current potential, which is multiplied by the number of groundings of the current \ac{pf}.
After all potentials are added up, \ac{tam} divides the potential by the number of overall groundings to obtain a mean potential.
To illustrate \cref{eq:mean}, consider a cluster with $3$ \acp{pf}.
The first \ac{pf} maps to the potentials $2$ and $1$ with $2$ groundings, the second maps to $3.9$ and $1.9$ with $5$ groundings, and the third maps to $8.1$ and $4$ with $1$ grounding.
To calculate the mean potential, \ac{tam} calculates for the first mapping $(2 \cdot 2+5 \cdot 3.9+1 \cdot 8.1)/8 = 3.95$ and for the second mapping $(2 \cdot 1+5 \cdot 1.9+1 \cdot 4)/8 = 1.9375$.
Thus, the mean \ac{pf} maps to $3.95$ and $1.9375$ with $8$ groundings.

To calculate variances of parfactors, \ac{tam} uses $rsim$ as the clusters have been built based on $rsim$.
The intuition behind the choice is that if two parfactors have a very small angle between their potentials, then the variance of the potentials would be close to $0$.
The variance increases with the angle between potentials. 
As the number of groundings influences the new potentials, we also include the number of groundings while calculating a variance as the function should reflect that a parfactor that represents more groundings has a greater weight than one parfactor with one grounding, i.e., semantically we have that factor more often and therefore, in the ground case the variance would be calculated more often.

After computing a mean parfactor $g_t^{i,k}$ for each cluster $K \in \mathbf{K}$  and an overall mean parfactor $g_t^{i,m}$ based on all parfactors in $\mathbf{K}$, ANOVA proceeds to compute the variation between groups, i.e., $MSG$, and within groups, i.e., $MSE$, using \cref{eq:rsim} and the groundings of \acp{pf}:

\begin{align*}
	MSG	&= \frac{1}{l-1} \sum_{K \in \mathbf{K}} gr(g_t^{i,k}) \cdot (rsim(g_t^{i,k}, g_t^{i,m}))^2 	\\
	MSE	&= \frac{1}{m-l} \sum_{K \in \mathbf{K}}\sum_{g_t^{i,j} \in K} gr(g_t^{i,j}) \cdot (rsim(g_t^{i,j}, g_t^{i,k}))^2		
\end{align*}
where $l = |\mathbf{K}|$, i.e., number of clusters, and $m = |gr(\mathbf{K})|$, i.e., number of overall groundings.
Computing $F = \frac{MSG}{MSE}$, ANOVA compares $F$ against a critical value $F_{crit}$, which depends on $\alpha$, $l-1$, and $m-l$ and can be looked up in a pre-computed table.
If $F \leq F_{crit}$, \ac{tam} accepts $H_0$ and discards the clustering.
In case \ac{tam} rejects $H_0$, i.e., $F > F_{crit}$, there is more difference between clusters than within clusters and \ac{tam} proceeds to merging parfactors. %in each cluster.

\subsubsection{Merging Parfactors}
The new parfactor for each cluster $K \in \mathbf{K}$ is the mean parfactor $g_t^{i,k}$ already computed by ANOVA.
\ac{tam} replaces $P$ in $G$ with the merged parfactors.
Then, \ac{tam} proceeds with the next partition, identifying and checking a clustering for the new partition, until all partitions are processed.
The result is a model whose parfactors are merged versions of the input model, partially restoring a lifted representation.
Given a forward message $m_t$, the output is a message that possibly contains fewer groups within logvars and thus, prevents ongoing splitting over time.

\subsubsection{Application Cycle}
As ANOVA may determine that the clustering is not fit enough, \ac{tam} may incur overhead if \ac{tam} cannot merge groups.
Therefore, in most cases, \ac{tam} should not be applied at every time step.
Normally, the model is slowly grounded over time with evidence, but if the groups behave similarly, which is the case due to the impact of the model, the reoccurring application of the model behaviour results in the potentials being similar enough for \ac{tam} to merge them.
Thus, based on how much evidence splits up the model, the interval of how often \ac{tam} should be used as a subroutine needs to be determined.

Next, we look at theoretical implications of \ac{tam}.

\section{Theoretical Analysis}
We show that \ac{tam} introduces an acceptable, unbiased, and bounded error and that \ac{tam} keeps reasoning polynomial.
%An extended version can be found in the appendix.

\begin{proposition}\label{pro:err}
	\ac{tam} errors are acceptable and unbiased.
\end{proposition}

Due to a density-based clustering, \ac{tam} clusters \acp{pf} with similar $\phi$'s.
ANOVA determines the fitness of clusterings to prevent unacceptable errors.
By accounting for groundings during merging, the error is unbiased.
    
Knowing that \ac{tam} produces acceptable and unbiased errors, let us have a look at theoretical bounds of the approximation error \ac{tam} introduces as well as whether groups with only slightly different evidence do converge, allowing \ac{tam} to keep reasoning polynomial.

\begin{theorem}\label{thm:bnd}
    \ac{tam} introduces a bounded error.
\end{theorem}
\begin{proof}\label{pf:bnd}
	A \ac{pdm} is a Markov process and $G_\rightarrow$ describes a temporal transitions model.
	Given the semantics of a \ac{pm}, $G_\rightarrow$ forms a stochastic transition model $Q$, which has a so-called minimal mixing rate $\gamma_Q \in\ ]0,1]$ \cite{boyen1998tractable}.
    $\gamma_{Q}$ is the minimal extent to which the model behaviour causes an approximation to converge to the true belief state while transitioning from one time step to the next.
	\ac{tam} approximates the belief state of the interface $\mathbf{I}_t$ and \ac{ldjt} computes the transition from $t$ to $t+1$.
	Thus, the approximation error $\delta$ is reduced by the factor $(1-\gamma_Q)$ with each transition.
    Assuming, that \ac{tam} introduces an error of at most $\delta$ for each time step, the expected error up to time step $t$ accumulates to $\delta + (1-\gamma_{Q}) \cdot \delta + ... + (1-\gamma_{Q})^{t-1} \cdot \delta = \sum_{i=0}^t \delta \cdot (1-\gamma_{Q})^i \leq \sum_{i=0}^\infty \delta \cdot (1-\gamma_{Q})^i = {\delta}/{\gamma_{Q}}$.
	For the last step, we apply the geometric series, i.e., $\sum_{i=0}^\infty \delta \cdot (1-\gamma_{Q})^i = {\delta}/1-(1-{\gamma_{Q}}) = {\delta}/{\gamma_{Q}}$ \cite{boyen1998tractable}.
	Thus, the error is indefinitely bounded by ${\delta}/{\gamma_Q}$.
\end{proof}

For \ac{tam} the significance check influences the approximation error $\delta$.
Before \ac{tam} merges \acp{pf} and thereby, approximates a belief state, \ac{tam} uses a significance check to determine the fitness of a proposed clustering.
Therefore, one can use the significance check to obtain a small $\delta$. 
Now, we prove that \ac{tam} keeps reasoning polynomial.

\begin{theorem}\label{thm:krp}
	\ac{tam} keeps reasoning polynomial.
\end{theorem}

\begin{proof}\label{pf:krp}
    Evidence introduces a discrepancy between two distributions of the same origin.
    The minimal mixing rate $\gamma_Q$ ensures that these two distributions converge again.
    Therefore, \ac{tam} will merge these two distributions at some point in time.
    Merging distributions ensures that \ac{ldjt} calculates a solution in polynomial time w.r.t.\ domains.
\end{proof}

Now, we use Thm. \ref{thm:krp} to restore an original representation.

\begin{corollary}\label{cor:true}
    Without new evidence, \ac{tam} obtains a fully lifted representation with the true belief state.
\end{corollary}

\begin{proof}\label{pf:true}
    During each transition from $t$ to $t+1$, $\gamma_Q$ ensures that approximated distributions converge to the true distribution as the distributions converge at least by the factor $(1-\gamma_Q)$.
    Thus, the approximated distributions converge to the true belief state without new evidence provided.
    Further, all groups have the same origin.
    Therefore, all groups converge to the same true belief state.
    Hence, \ac{tam} can merge all groups and thereby, again obtain a fully lifted representation at some point in time.
\end{proof}

Thus, \ac{tam} solves the \ac{krp} problem.
Since the underlying distributions of $\phi$'s converge, \ac{tam} is able to merge $\phi$'s, allowing \ac{tam} to keep reasoning polynomial.
Further, \ac{tam} introduces a bounded, unbiased, and acceptable error.

\section{Evaluation}

For the evaluation, we compare runtimes of \ac{ldjt} with and without \ac{tam} and have a look at the introduced error.
We use the model $G^{ex}$ with $|\mathcal{D}(X)| = 100$ and divide these 100 persons equally into symmetry groups, where members of each group behave identically over time.
For one time step, each symmetry group has the same evidence, but the evidence can change from one time step to the next.
To break symmetries within a group, evidence may be missing with a probability of $0.1$ for each person.
We split $D(X)$ into $2$ to $10$ symmetry groups and generate evidence for $20$ time steps.
For each symmetry group $i$, \ac{ldjt} answers $A_{t+\pi}(x_i)$ for $\pi = \{0,5,10\}$ in each time step $t$ for all $20$ time steps.

We vary $\epsilon$ and the interval $I$ of how often \ac{ldjt} applies \ac{tam}.
$\alpha$ is fixed to $0.005$. 
Based on the problem at hand, an appropriate $\alpha$ needs to be determined in advance \cite{benjamin2018redefine}.
The three options we evaluate are, from conservative to aggressive:
\begin{enumerate*}[label*=\arabic*)]
    \item $I = 5$, $\epsilon = 5 \cdot 10^{-14}$,
    \item $I = 5$, $\epsilon = 5 \cdot 10^{-2}$, and
    \item $I = 2$, $\epsilon = 5 \cdot 10^{-2}$. 
\end{enumerate*}
\ac{tam} with Option 1 mostly merges \acp{pf} that only differ in a scaling factor.
\ac{tam} with Options 2 and 3 also merges \acp{pf} that slightly differ in their ratio.
With $I = 2$, \ac{ldjt} calls \ac{tam} every other time step, and with $I = 5$ every fifth time step.

\Cref{fig:times_overall} shows runtimes of \ac{ldjt} without \ac{tam} and with \ac{tam} for the three options.
The number of symmetry groups is plotted on the x-axis.
With more symmetry groups, evidence can ground the model faster over time.
Thus, the runtimes correlate to the number of groups.
For $5$ symmetry groups, \ac{ldjt} without \ac{tam} takes about twice as long as \ac{ldjt} with \ac{tam} using the conservative option (1), answering $300$ queries for the $20$ time steps.
However, for $8$ symmetry groups, \ac{ldjt} without \ac{tam} is slightly faster.
As merging depends on evidence, which here is randomly generated, \ac{tam} may not always be able to trade off its overhead.
\ac{tam} with Option 2 merges more \acp{pf}.
Hence, every fifth time step, \ac{ldjt} answers queries on fewer groups, which are then again split up by evidence. %fifth?
With the most aggressive option (3), \ac{ldjt} applies \ac{tam} every other time step and thus answers queries on highly lifted models.

In summary, even by only merging \acp{pf} that hardly differ, \ac{tam} merges enough \acp{pf} to improve runtimes of \ac{ldjt}.
\ac{tam} with Options $2$ and $3$ improves runtimes of \ac{ldjt} significantly.
Overall, \ac{tam} is able to save runtime of \ac{ldjt} of up to $2$ orders of magnitude.
Knowing that \ac{tam} can significantly improve the performance of \ac{ldjt}, we look at the costs of the speed up, namely the introduced error.

\Cref{tab:error} shows the error in the marginals for $10$ symmetry groups for the most aggressive option, when performing filtering, $2$ time step prediction, and $4$ time step prediction for each instance and each time step.
For filtering queries, the error is already negligible and decreases for prediction queries.
Thus, the empirical evaluation underscores that \ac{tam} can keep reasoning polynomial, introducing only a negligible error.
Further, the error converges to the true belief state without new evidence as the prediction queries show.
Next, we take a look at the significance check.

\begin{figure}[t]
  \includegraphics[width=0.95\linewidth]{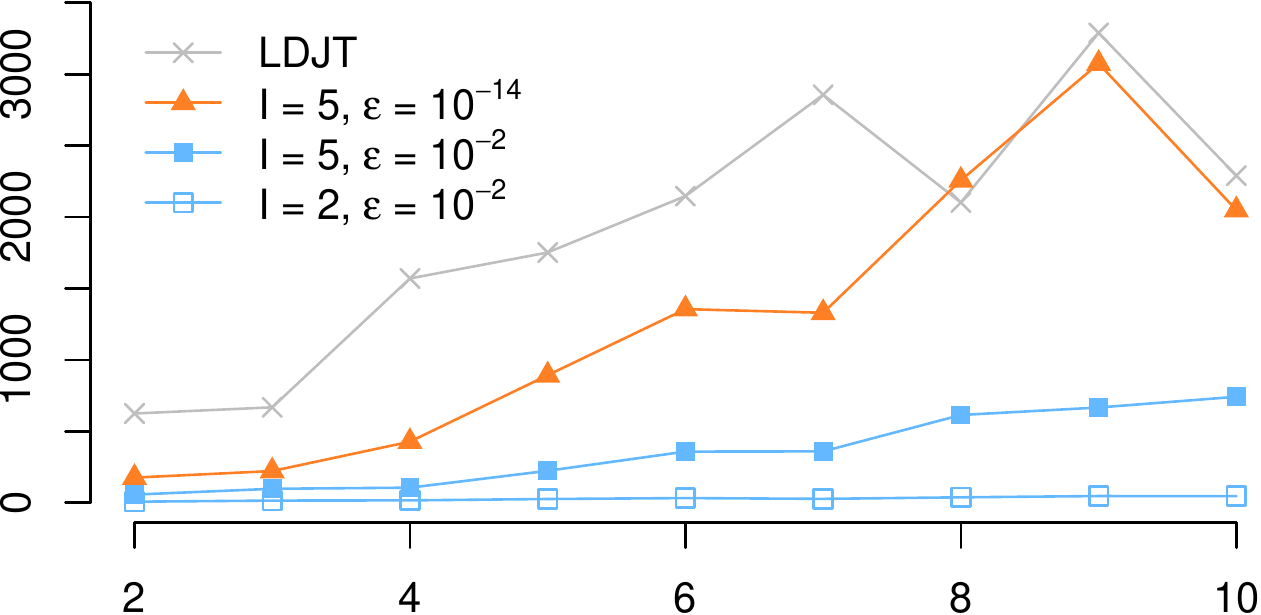}
  \caption{Runtimes [seconds], x-axis: \#symmetry groups}
  \label{fig:times_overall}	
\end{figure}

\begin{table}\centering
%    \begin{center}
    \resizebox{.9995\linewidth}{!}{
    \begin{tabular}{cccc}
        \toprule
        $\pi$ & Max & Min & Average \\
        \midrule
        $0$ & 
        $0.0001537746121$   & 
        $0.0000000001720$ & 
        $0.0000191206488$\\     
        $2$ & 
        $0.0000000851654$   & 
        $0.0000000000001$  & 
        $0.0000000111949$\\        
        $4$ & 
        $0.0000000000478$& 
        $0$ & 
        $0.0000000000068$\\      
        \bottomrule     
    \end{tabular}
    }
%    \end{center}
    \caption{Introduced error; $I = 2$, $\epsilon = 5 \cdot 10^{-2}$, $10$ groups}
    \label{tab:error}
\end{table}

\begin{table}\centering
%    \begin{center}
    \resizebox{.9995\linewidth}{!}{
    \begin{tabular}{cccc}
        \toprule
         & Max & Min & Average \\
        \midrule
        w & 
        $0.0002259927071$ &
        $0.0000000000000$ &
        $0.0000104567643$\\
        %$0.0000191206488$\\     
        w/o & 
        $0.0002260554389$   & 
        $0.0000000000168$  & 
        $0.0000137870835$\\           
        \bottomrule     
    \end{tabular}
    }
%    \end{center}
    \caption{Introduced error; with and without significance test}
    \label{tab:sig}
\end{table}

To empirically evaluate the significance check, we run \ac{ldjt} with \ac{tam} on a model once with and once without the significance check.
\Cref{tab:sig} shows the introduced errors for these runs.
The maximum error hardly differs between the two runs, which is is due to the error being bounded.
Further, the minimum error is lower with the significance check as the significance check does not accept all proposed clusters. 
Discarding a clustering and thus, not following through with another approximation, the current approximation and the true belief state continue to converge based on the mixing rate.
Lastly, the average error without the significance check is around $32\%$ higher.
Even though in this case both average errors are negligible on an absolute scale, the average error on a relative scale without the significance check does increase significantly.

Overall, we show empirically that \ac{tam} does not introduce any unacceptable error due to the significance check and that \ac{tam} keeps reasoning polynomial for \ac{ldjt}.

\section{Conclusion}\label{sec:conc}
Evidence often grounds a temporal model over time. %as evidence within symmetry groups may slightly differ from one time step to the next.
Consequently, inference runtimes suffer.
Thus, the idea is to use approximate symmetries to restore a lifted representation and thereby, keep reasoning polynomial by taming evidence.
To the best of our knowledge, we present the first approach solving the \ac{krp} problem for temporal relational probabilistic models, which can be used within any (exact or approximate) temporal inference algorithm.
The main idea is that instances of \acp{pf} with similar ratios between potentials behave similarly.
To merge \acp{pf}, \ac{tam} uses a message \ac{ldjt} sends between time steps as this message is smaller than the model and causes splits in the next time step. 
To identify similar instances, \ac{tam} uses density-based clustering with the cosine similarity as a distance measure, which captures similarity of potentials.
\ac{tam} applies ANOVA to the clustering result to check if the cluster means significantly discriminate the clusters. 
We show that \ac{tam} can merge \acp{pf} as their distributions converge and that \ac{tam} introduces a bounded error.
Additionally, the approximated distributions converge to the true distributions and \ac{tam} can obtain a fully lifted representation again without new evidence. %since the influence of the model outweighs a slight difference in evidence over time. 
Empirical results show that \ac{ldjt} with \ac{tam} significantly outperforms \ac{ldjt} without \ac{tam}. 
The results support our analysis that \ac{tam} retains a lifted solution, while keeping the introduced error negligible. %\linebreak
Hence, \ac{ldjt} with \ac{tam} produces fast and precise results.

Future work includes how to approximate evidence \cite{van2013complexity} to cause fewer splits in temporal models as well as learning temporal models.

\subsubsection*{Acknowledgement}
This research originated from the Big Data project being part of Joint Lab 1, funded by Cisco Systems, at the centre COPICOH, University of Lübeck

\bibliographystyle{aaai}
%\bibliography{tex/bib}
\bibliography{bib}

\end{document}